\documentclass[11pt]{article}
\usepackage{fullpage}
\usepackage{graphicx}
\usepackage{amssymb}
\usepackage{epstopdf}
\usepackage{amsmath,amscd}
\usepackage{amsthm}
\usepackage{enumerate}
\usepackage{url,verbatim,soul}
\usepackage{subfig}
\usepackage{amsmath,amssymb,amsfonts,mathabx,setspace}
\usepackage[colorlinks,linkcolor=blue, citecolor=blue,urlcolor=blue]{hyperref}
\theoremstyle{plain}
\newtheorem{theorem}{Theorem}
\newtheorem{definition}[theorem]{Definition}
\newtheorem{lemma}[theorem]{Lemma}
\newtheorem{proposition}[theorem]{Proposition}

\newtheorem{example}[theorem]{Example}

\newtheorem{remark}[theorem]{Remark}

\newcommand{\innerp}[2]{\langle #1,#2 \rangle}
\newcommand{\hypspace}{\mathcal{H}}

\newcommand\RR{{\mathbb R}}

\newcommand\NN{{\mathbb N}}

\renewcommand\ell{l}

\newcommand\CC{\mathbb{C}}

\newcounter{mycount}

\newcommand{\FL}[1]{\textcolor{blue}{{#1}}}

\newcommand\bb{{\mathbf b}}
\newcommand\bX{{\mathbf X}}
\newcommand\bx{{\mathbf x}}
\newcommand\bY{{\mathbf Y}}
\newcommand\by{{\mathbf y}}
\newcommand\bB{{\mathbf B}}
\newcommand\bW{{\mathbf W}}
\newcommand\br{{\mathbf r}}
\newcommand\R{{\mathbb R}}
\newcommand\E{{\mathbb E}}

\numberwithin{equation}{section}
\numberwithin{theorem}{section}
\numberwithin{figure}{section}

\title{On the coercivity condition in the learning of \\ interacting particle systems}

\date{}
 \author{Zhongyang Li\thanks{Department of Mathematics, 
University of Connecticut, 
Storrs, CT, USA 06269-1009,
zhongyang.li@uconn.edu;}\,  
and Fei Lu\thanks{Department of Mathematics,  Johns Hopkins University, Baltimore, MD, USA, 21218; feilu@math.jhu.edu;}
 }

\begin{document}
\maketitle 
\begin{abstract}
In the inference for systems of interacting particles or agents, a coercivity condition ensures identifiability of the interaction kernels, providing the foundation of learning. We prove the coercivity condition for stochastic systems with an arbitrary number of particles and a class of kernels such that the systems of relative positions are ergodic. When the system of relative positions is stationary, we prove the coercivity condition by showing the strictly positive definiteness of an integral kernel arising in the learning; and for non-stationary case, we show that the coercivity condition holds true when the time is large based on a perturbation argument. 
\end{abstract}

\tableofcontents

\section{Introduction}
Consider the inference of $\phi$ in the stochastic system of interacting particles
\begin{equation}\label{eq:sys1st}
 d{\bX_i^t} =\frac{1}{N} \sum_{1\leq j\leq N, j\neq i} \phi(|\bX_{j}^t - \bX_i^t|) ( \bX_{j}^t - \bX_i^t)dt+  d\bB_i^t, \quad \text{for $i = 1, \ldots, N$}, 
\end{equation}
where $\bX^t:= (\bX_1^t,\dots, \bX_N^t)\in \RR^{dN}$ represents the positions of particles at time $t$,  and $\bB^t = (\bB_1^t,\dots,\bB_N^t)$ is a standard Brownian motion on $\RR^{dN}$ representing the environmental noise. Here $|\cdot|$ denotes the Euclidean norm of vectors and $\phi:\R^+\to \R$ is referred as \emph{interaction kernel}. 
Examples of such systems range from particles system in physics and chemistry to opinion dynamics in social science (see \cite{MT2014,malrieu2003convergence,LZTM19,LMT21jmlr} and the reference therein).  
In applications of such systems, the first task is to learn the interaction kernel $\phi$ from data. Parametric inference of the interaction kernel has been studied in \cite{chen2020maximum} when the kernel $\phi$ is a constant and in \cite{sharrock2021parameter} for McKean-Vlasov (mean-field) SDEs.  Nonparametric inference of general interacting kernel has been studied in \cite{LZTM19,LMT21jmlr,miller2020learning} for deterministic first- and second-order systems, in \cite{LMT20} for stochastic first-order systems, and in \cite{BFHM17,LangLu20} for mean-field equations. A fundamental issue is the identifiability of the interaction kernel because its values are underdetermined from data.  


A \emph{coercivity condition} is found to be sufficient for the identifiability, making the inverse problem of estimating $\phi$ well-defined and well-posed on finite hypothesis spaces \cite{LLMTZ21,LMT20,LZTM19,LMT21jmlr}. Note that the variable of $\phi$ is the pairwise distances $|\br_{ij}^t|$ with 
 \[
\br_{ij}^t: =\bX_i^t-\bX_{j}^t. 
\]

\begin{definition}[Coercivity condition]\label{def:coercivity} 
The system \eqref{eq:sys1st} on $[0,T]$ is said to satisfy a \emph{coercivity condition} on a finite dimensional linear space $\hypspace$ if  there is a constant $c_{\hypspace, T} >0$ such that 
\begin{equation} \label{eq:c_t1}
\widebar I_T(h) :=\frac{1}{T} \int_0^T \E\left[h(|\br_{12}^t|)h(|\br_{13}^t|) \frac{\innerp{\br_{12}^t}{\br_{13}^t}}{|\br_{12}^t||\br_{13}^t|}\right]dt \geq c_{\hypspace, T} \frac{1}{T}\int_0^T \E[h(|\br_{12}^t|)^2]dt 
\end{equation}
for any $h\in \hypspace$ such that $\int_0^T \E[h(|\br_{12}^t|)^2]dt <\infty$. When $\hypspace$ is infinite dimensional, we say that the coercivity condition holds on $\hypspace$ if it holds on any finite dimensional linear subspace of $\hypspace$.   
\end{definition}


This coercivity condition is conjectured with intuition from Gaussian initial distributions and is numerically verified in \cite{LMT20,LZTM19,LMT21jmlr}. It is partially proved in \cite{LLMTZ21} when $N=3$ or when the system is linear and $\{\br_{ji}^t\}$ are stationary. Since the coercivity condition plays a crucial role in the inverse problem with a wide range of applications, it is of paramount importance to further understand the coercivity condition for systems that are nonlinear and with $N> 3$.

In this study, we prove the coercivity condition for systems with an arbitrary number of particles and any interaction kernel $\phi(r) =\frac{\Phi_0'(r)}{r}$ with $\Phi_0$ in the form:   
\begin{equation}\label{eq:Phi_0def}
 \Phi_0(r)=(a+r^\theta)^{\gamma} \text{ with } a\geq 0, \, \theta\in (1,2],\, \gamma\in (0, 1] \text{ such that } \gamma\theta>1, 
\end{equation}
 and when $T$ is large. 
Our result completes those in \cite{LLMTZ21}, where the coercivity condition is proved either for $(\theta, \gamma)=(2,1)$ (i.e., linear systems) or for systems with $N=3$ particles, and only when $\{\br_{ji}^t\}$ are stationary.  
We achieve it in two steps: (i) we first show that the system of relative positions $(\br_{12}^t, \br_{13}^t),\ldots, \br_{1N}^t)$ converges to an explicit stationary density at a polynomial rate in time, based on a convergence result for systems with non-uniformly convex potentials in \cite{TV00} (see Section \ref{sec:existStationary}--\ref{sec:convSD});  (ii) we then prove the coercivity condition for the system with an arbitrary $N$ at the stationary density; 
and for the non-stationary case when $T$ is large based on a perturbation argument.   

The coercivity condition is intrinsically connected to strictly positive definiteness of integral kernels. More specifically, we write $\widebar I_T(h)$ in \eqref{eq:c_t1} as 
\begin{equation}\label{eq:I_Integral_kernel}
\widebar I_T(h) = \int_{\R^d}\int_{\R^d} h(|u|)h(|v|) K_T(u,v)dudv, \quad \text{ where }
K_T(u,v) = \frac{\innerp{u}{v}}{|u||v|} \widebar{p}_T(u,v)
\end{equation}
with $\widebar{p}_T(u,v) =\frac{1}{T} \int_0^T p_t(u,v) dt$, where $p_t(u,v)$ is the probability density of $(\br_{12}^t,\br_{13}^t)$.  
Then, equation  \eqref{eq:c_t1} holds on any finite dimensional space $\hypspace$ as long as $K_T$ is a strictly positive definite. We will prove that $K_T(u,v)$ and $\widebar{p}_T(u,v)$, among a class integral kernels related to the interacting particle system, 
 are strictly positive definite based on M\"untz type theorems. 

Positive definite kernels play an increasingly prominent role in many applications, in particular in statistical learning with reproducing kernel Hilbert space (RKHS) \cite{cucker2002mathematical,yang2017randomized}. Our results provides a new class of strictly positive definite integral kernels, which have a potential use in sampling and kernel-based learning. Also, our technique 
can be used to establish strict positive definiteness of general integral kernels in learning problems. 

The organization of the paper is as follows: in Section \ref{sec:existStationary}, we prove that the system of relative positions $(\br_{12}^t,\br_{13}^t,\ldots,\br_{1N}^t)\in \R^{d(N-1)}$ has a stationary density, and we prove in Section \ref{sec:convSD} that the system converges to the stationary density at a polynomial rate in time.  In Section \ref{sec:PD_inv} we show that the coercivity condition holds true at the stationary density by showing that the related integral kernels are strictly positive definite. In Section \ref{sec:coercivityMean} we show that the coercivity condition holds true for non-equilibrium systems, under additional restriction on the hypothesis space.  
We list in Appendix \ref{sec:append} the preliminary theory on  positive definite kernels. 

\section{Existence of a stationary density}\label{sec:existStationary}
We show that the system of relative positions $(\br_{12}^t,\br_{13}^t,\dots,\br_{1N}^t)$ with  $\br_{ij} = \bX^t_i-\bX^t_j$ has a stationary density. Note that the original system does not have a stationary distribution on $\RR^{dN}$, because 
the center position of the particles, $\bX_c(t) =:\frac{1}{N}\sum_{i=1}^N \bX_i^t$, satisfies 
$ d\bX_c = \frac{1}{N}\sum_{i=1}^N \bB_i^t$, 
 which follows from $\frac{1}{N}\sum_{i=1}^N \nabla_{\bx_i} J_{\Phi}(\bx) =0$ due to the symmetry of the interaction between the particles.  
  
We rewrite the system \eqref{eq:sys1st} in the form of a gradient system
\begin{equation}\label{gs}
d\bX^t=-\nabla J_{\Phi}(\bX^t)dt +d\mathbf{B}^t,
\end{equation}
 where the energy potential $J_{\Phi}:\RR^{dN}\rightarrow\RR$ depends on the pairwise distances: 
\begin{eqnarray*}
J_{\Phi}(\bX)=\frac{1}{2N}\sum_{i,j=1}^{N}\Phi(|\bX_i-\bX_j|),\ \bX\in \RR^{dN}, 
\end{eqnarray*}
with $\Phi:[0,\infty)\rightarrow \RR$ satisfying  $\phi(r) = \frac{\Phi'(r)}{r}$ and without lost of generality, we set $\Phi(0)=0$.

Our study for the system of relative positions 
is similar to the macro–micro decomposition of the system in \cite{malrieu2003convergence}.  But the relative positions are directly relevant to inference: the joint distribution of $(\br^t_{12},\br^t_{13})$ is what we need for the coercivity condition. In the following, we will first show that the system of relative positions is a gradient system with an additive noise, thus, its stationary density is explicit, so is the marginal distribution of $(\br^t_{12},\br^t_{13})$.

\begin{theorem}\label{thm:stationaryDensity}
For $1\leq i<j\leq N$ and the process $\bX^t$ satisfying the system \eqref{gs}, denote
\begin{eqnarray*}
\mathbf{r}_{ij}^t=\bX_i^t-\bX_j^t,\ \mathrm{and}\ r_{ij}^t=|\mathbf{r}_{ij}^t|.
\end{eqnarray*}
Suppose that the function $H:\RR^{d(N-1)}\to \RR$ 
\begin{eqnarray}
H(\mathbf{r})=J_{\Phi}(X)=\frac{1}{N}\sum_{2\leq j\leq N}\Phi(| \mathbf{r}_{1j}| )+\frac{1}{N}\sum_{2\leq i< j\leq N}\Phi(| \mathbf{r}_{1i}-\mathbf{r}_{1j} |)\label{hhr}
\end{eqnarray}
satisfies $\int_{\RR^{d(N-1)}}e^{H(\br)}d\br <\infty$. 
Then, the process $\br^t=(\br_{12}^t,\br_{13}^t,\dots,\br_{1N}^t)^T$ has a stationary density, and the process
\begin{eqnarray} \label{eq:r2Y}
\mathbf{Y}^t=S^{-1}\mathbf{r}^t
\end{eqnarray} 
satisfies a gradient system: 
\begin{eqnarray}\label{ggs}
d\mathbf{Y}^t=-\nabla_{Y^t}H(S\mathbf{Y}^t)+d\mathbf{W}^t, 
\end{eqnarray}
where $\mathbf{W}^t$ is a standard Brownian motion on $\RR^{(N-1)d}$, the matrix $S\in \RR^{(N-1)d\times (N-1)d}$ is an invertible matrix satisfying 
\begin{eqnarray} 
S S^{T} = A :=\left(\begin{array}{cccc}2I_d& I_d&\cdots &I_d\\ I_d & 2I_d&\cdots& I_d\\ \cdots\\ I_d& I_d&\cdots & 2I_d\end{array}\right). \label{as}
\end{eqnarray}
\end{theorem}


\begin{remark}
The matrix $A$ in \eqref{as} has an eigenvalue $\lambda = N$ with multiplicity $d$ and an eigenvalue $\lambda= 1$ with multiplicity $(N-2)d$. Since $A$ is symmetric positive definite, we can find an invertible matrix $S$ satisfying \eqref{as}, for example, 
\begin{eqnarray}
S=\left(\begin{array}{cccccc}\sqrt{2}I_d&0&0&\ldots&0&0\\\frac{\sqrt{2}}{2}I_d&\frac{\sqrt{6}}{2}I_d&0&\ldots&0&0\\\frac{\sqrt{2}}{2}I_d&\frac{\sqrt{6}}{6}I_d&\frac{2\sqrt{3}}{3}I_d&\ldots&0&0\\\ldots&&&&\\\sqrt{\frac{1}{2}}I_d&\sqrt{\frac{1}{2\cdot 3}}I_d&\sqrt{\frac{1}{3\cdot 4}}I_d&\ldots&\sqrt{\frac{1}{(N-2)(N-1)}}I_d&\sqrt{\frac{N}{N-1}}I_d\end{array}\right).\label{sm}
\end{eqnarray}
\end{remark}

\bigskip
Before proving Theorem \ref{thm:stationaryDensity}, we first apply it to obtain the invariant density of $(\br^t_{12},\br^t_{13})$. 
\begin{proposition}Let $\bX^t\in \RR^{dN}$ be the solution of \eqref{gs}. The invariant density for $(\br^t_{12},\br^t_{13})$ has the following form
\begin{equation}\label{puv}
p_\infty(u,v)=\frac{1}{Z}f(u,v)e^{-\frac{2}{N}\left[\Phi(|u|)+\Phi(|v|)+\Phi(|u-v|)\right]},
\end{equation}
where \,\,\quad 
\begin{eqnarray}
&&f(u,v)=\int
\label{fuv} e^{-\frac{2}{N}\left[\sum_{4\leq i<j}^N\Phi(|\mathbf{r}_{1i}-\mathbf{r}_{1j}|)+\sum_{l=4}^N\left[\Phi(|\mathbf{r}_{1l}|)+\Phi(|u-\mathbf{r}_{1l}|)+\Phi(|v-\mathbf{r}_{1l}|)\right]\right]}     d\mathbf{r}_{14}\ldots\mathbf{r}_{1N}, \,\,\quad 
\end{eqnarray}
and $Z$ is a normalizing constant given by
\begin{eqnarray*}
Z=\int_{\RR^{d}}\int_{\RR^{d}}f(u,v)e^{-\frac{2}{N}\left[\Phi(|u|)+\Phi(|v|)+\Phi(|u-v|)\right]}   dudv
\end{eqnarray*}
\end{proposition}

\begin{proof} 
By Lemma \ref{lemma:GradS_inv} below, the probability density
\begin{equation}\label{eq:inv_pdf}
p_{_\bY}(\mathbf{y})=\frac{1}{Z_N}e^{-2 H(S\mathbf{y})}
\end{equation}
with $Z_N:=\int_{\RR^d}\ldots \int_{\RR^d} e^{-2 H(S\mathbf{y})}   d\mathbf{y}_1 \ldots d\mathbf{y}_{N-1}$, 
is an invariant density for \eqref{ggs}. Since $\br = S\by$, by integrating $p_{_\bY}(\by)$ with respect to $\mathbf{r}_{14},\ldots, \mathbf{r}_{1N}$, we obtain the invariant  density for  $(\br^t_{12},\br^t_{13})$.
\end{proof}

\begin{proof}[Proof of Theorem \rm{\ref{thm:stationaryDensity}}]
Note that with the notation $\mathbf{r}_{ij}^t=\bX_i^t-\bX_j^t$ and $\ r_{ij}^t=|\mathbf{r}_{ij}^t|$, the system \eqref{gs} is equivalent to
\begin{eqnarray}
d\mathbf{r}_{1i}^t=-\frac{1}{N}\left[2\phi(r_{1i}^t)\mathbf{r}_{1i}^t+\sum_{2\leq j\leq N,j\neq i}\phi(r_{1j})\mathbf{r}_{1j}^t+\phi(r_{ji})\mathbf{r}_{ji}^t\right]dt+d(\mathbf{B}_1^t-\mathbf{B}_i^t), \label{sd}
\end{eqnarray}
for $2\leq i\leq N$.

The process $(\mathbf{B}_1^t-\mathbf{B}_2^t,\mathbf{B}_1^t-\mathbf{B}_3^t,\ldots,\mathbf{B}_1^t-\mathbf{B}_N^t)^{T}$ (where $T$ means transpose) is a $d(N-1)$-dimensional Brownian motion with mean 0 and covariance $A$. Then, it can be written as $S\bW^t$, with $(\bW^t, t\geq 0)$ being a standard Brownian motion on $\RR^{d(N-1)}$. 


Note that for any $2\leq j\leq N$ and $j\neq i$, we have $\mathbf{r}^t_{ji}=\mathbf{r}_{1i}^t-\mathbf{r}^t_{1j}$. Then, with $\bb(\mathbf{r}^t) =(\bb_2(\mathbf{r}^t),\ldots,\bb_N(\mathbf{r}^t))^T, 
$
where for $2\leq i\leq N$ we denote
\begin{eqnarray*}
\bb_i (\mathbf{r}^t) =\frac{1}{N}\left[2\phi(r_{1i}^t)\mathbf{r}_{1i}^t+\sum_{2\leq j\leq N,j\neq i}\phi(r_{1j}^t)\mathbf{r}_{1j}^t+\phi(r_{ji}^t)\mathbf{r}_{ji}^t\right],
\end{eqnarray*}
 we may write the system (\ref{sd}) as follows:
\begin{eqnarray}
d\mathbf{r}^t=-\bb(\mathbf{r},t)dt+Sd\mathbf{W}^t\label{sd1}.
\end{eqnarray}
Multiplying both sides of (\ref{sd1}) by $S^{-1}$, we obtain
\begin{eqnarray}
d\mathbf{Y}^t=-S^{-1}\bb(\mathbf{r}^t)dt+d\mathbf{W}^t\label{sd2}
\end{eqnarray}

To write it as a gradient system, note that for $2\leq i\leq N$, we have
\begin{eqnarray*}
\bb_i(\mathbf{r})=2\nabla_{\mathbf{r}_{1i}}H(\mathbf{r})+\sum_{2\leq j\leq N,j\neq i}\nabla_{\mathbf{r}_{1j}}H(\mathbf{r})
\end{eqnarray*}
Hence,
\begin{eqnarray}
\bb(\mathbf{r})=\bb(S\mathbf{Y})=A\nabla_{\mathbf{r}} H(\mathbf{r}^t)=A[S^{-1}]^T \nabla_{\mathbf{Y}}H(S\mathbf{Y}).\label{me}
\end{eqnarray}
Plugging (\ref{me}) to (\ref{sd2}), and using the equation (\ref{as}), we obtain the gradient system \eqref{ggs}. 

Then, by Lemma \ref{lemma:GradS_inv} below, the process $(\bY^t)$ defined by the system \eqref{ggs} has a stationary density, and so does the process $(\br^t)$, which is a linear transformation of $\bY^t$. 
\end{proof}
\begin{lemma}\label{lemma:GradS_inv}
Suppose that $\nabla H:\RR^n\to \RR^n$ is locally Lipschitz and that $Z= \int_{\RR^n}e^{-\frac{2}{\sigma^2}H(x)} dx<\infty$ with $\sigma>0$. Then $p(x)= \frac{1}{Z} e^{-\frac{2}{\sigma^2}H(x)}$ is an invariant density to gradient system 
\begin{equation*}
dX_t= -\nabla H(X_t)dt +\sigma dB_t,
\end{equation*}
where $(B_t) $ is an n-dimensional standard Brownian motion. 
\end{lemma}
\begin{proof}
It follows directly by showing that $p(x)$ is a stationary solution to the backward Kolmogorov equation, i.e. 
\[\frac{\sigma^2}{2} \Delta p + \nabla \cdot (p \nabla H) = 0.  
\]
\end{proof}

\section{Convergence to the stationary density}\label{sec:convSD}
We show in this section that for a class of potentials $\Phi$, the gradient system \eqref{ggs} is ergodic, converging to the stationary density at a polynomial rate in time.  

\begin{theorem}\label{thm:Ergodicity}
With potential $\Phi =\Phi_0$ given by \eqref{eq:Phi_0def}, the system \eqref{ggs} is ergodic, converging polynomial in time to the stationary density. 
More precisely, if the initial condition $\bY^0$ has a probability density $p_0(\by)$ satisfying 
\begin{equation}\label{eq:IC}
\begin{aligned}
 & \E[H(S\bY^0) \log p_0(\bY^0) ] < \infty, \\ 
 & \E[(1+ |\bY^0|^2)^{\frac{s}{2}}] < \infty
\end{aligned}
\end{equation}
for some $s\geq 2$. Then, the density $p_{_{\bY^t}}(\by)$ of $\bY^t$ converges to $p_{_\bY}(\by)$ defined in \eqref{eq:inv_pdf}: 
\begin{equation}\label{eq:conv2inv}
 \|p_{_{\bY^t}}(\by) - p_{_\bY}(\by)\|_{L^1}^2 \leq \frac{C}{t^\kappa}
\end{equation}
for a constant $C$ independent of $t$, and 
\begin{align}\label{eq:kappa}
    \kappa = \frac{s-2}{2-\theta\gamma}. 
\end{align} 
\end{theorem}
\begin{proof}[{\bf Proof of Theorem \ref{thm:Ergodicity} }]
The theorem follows directly from Theorem \ref{ttv}, Proposition \ref{prop_ergodicity} and  Example \ref{exmp_1}-\ref{exmp_2}, which are proved below. 
\end{proof}

When $\theta =2$ and $\gamma=1$, the potential $\Phi(r) = a+r^2$ lead to a linear system and $(\bY^t)$ is an Ornstein-Uhlenbeck process, and we have exponential convergence. In general, when $\Phi$ is uniformly convex, i.e. $\mathrm{Hess}_x \Phi(|x|) \geq \lambda I_d$, one has exponential convergence to the equilibrium for the entropy from $p_{_{\bY^t}}$ to $p_{_{\bY}}$ as in \cite{CMV03,malrieu2003convergence}. Here we focus on the $L^1$ distance, which will be needed in proving the coercivity condition in Section \ref{sec:coercivityMean}. In particular, the following convergence in $L^1$ of the marginal density of $(\bX_1^t-\bX_2^t, \bX_1^t-\bX_3^t)$ is needed. 
\begin{proposition}\label{prop:Conv_poly}
Let $\bX^t$ be a solution to the system \eqref{gs} with potential $\Phi =\Phi_0$ given by \eqref{eq:Phi_0def}, and with initial condition satisfying \eqref{eq:IC}. Denote by $p_t(u,v)$ the density of $(\bX_1^t-\bX_2^t, \bX_1^t-\bX_3^t)$. Then $p_t(u,v)$ converges to the stationary density $p_\infty(u,v)$ in \eqref{puv} at a polynomial rate in $t$:
\begin{align*}
    \|p_t -  p_\infty\|_{L^1} = \int_{\R^{2d}} | p_t(u,v) - p(u,v)| dudv \leq C t^{-\kappa/2}, 
\end{align*}
where $C$ is a constant independent of $t$ and $\kappa$ is given in \eqref{eq:kappa}. 
\end{proposition}
\begin{proof}
By Theorem \ref{thm:stationaryDensity}, the process $\bY^t = S^{-1}\br^t$ with $\br^t=(\bX^t_1-\bX^t_2,\bX^t_1-\bX^t_3,\dots,\bX^t_1-\bX^t_N)$ satisfies the system \eqref{ggs}. Let $p_{\br^t}$ be the density of $\br^t$, and let $p_\br$ be the corresponding stationary density. Then, 
\begin{align*}
&\int_{\RR^{2d}} \left|p_t(u,v)-p_\infty(u,v)\right|dudv \\
=&\int_{\RR^{2d}}\left|\int_{\RR^{d(N-3)}}p_{\br^t}(u,v,\br_{14},\ldots,\br_{1N})-p_{\br}(u,v,\br_{14},\ldots,\br_{1N})d\br_{14}\ldots d\br_{1N}\right|dudv\\
\leq &\int_{\RR^{d(N-1)}}\left| p_{\br^t}(u,v,\br_{14},\ldots,\br_{1N})-p_{\br}(u,v,\br_{14},\ldots,\br_{1N})\right|d\br_{14}\ldots d\br_{1N}dudv\\
= & \int_{\RR^{d(N-1)}}\left|p_{_{\bY^t}}(\by) - p_{_\bY}(\by)\right| d\by  
\leq |\mathrm{det}S^{-1}| \sqrt{C} t^{-\kappa/2},
\end{align*}
where the last equation follows from \eqref{eq:conv2inv} of Theorem \ref{thm:Ergodicity}. 
\end{proof}

The proof of Theorem \ref{thm:Ergodicity} is based on the following theorem in \cite{TV00}.

\begin{theorem}[{\cite[Theorem 3]{TV00}}]\label{ttv} 
Let $x\in\RR^n$. Assume $W\in W_{\mathrm{loc}}^{2,\infty}$ satisfies:
\begin{itemize}
\item[(1)] $\int_{\RR^n}e^{-W(x)}dx =1 $; and
\item[(2)] there exist $U:\R^n\to \R$ and constant $a,b>0$ such that  for all $x\in \RR^n$,
\begin{eqnarray*} 
U(x)-a\leq W(x)\leq U(x)+b, 
\end{eqnarray*}
and there exist $c>0$ and $\alpha\in (0,2)$, such that the matrix
\begin{eqnarray*}
\mathrm{Hess}[U(x)]-c(1+|x|)^{\alpha-2}I_n
\end{eqnarray*}
is positive semi-definite, in which $Hess[U(x)]$ is the Hessian matrix for $U(x)$; 
and
\item[(3)] there exist $\beta>0$, $C_0,C_1>0$ and $x\in \RR^n$ satisfying
\begin{eqnarray*}
\nabla W(x)\cdot x\geq C_1|x|^{\beta}-C_0.
\end{eqnarray*}
\end{itemize}

Let $f_0$ be a probability density such that
\begin{eqnarray}
\int_{\RR^n}f_0(x)[\log f_0(x)+W(x)]dx<\infty;\ \mathrm{and}\label{c1}\\
\int_{\RR^n}f_0(x)(1+|x|^2)^{\frac{s}{2}}dx<\infty; \mathrm{for\ some\ }s>2.\label{c2}
\end{eqnarray}
Let $f(t,\cdot)$ be a smooth solution of the Fokker-Planck equation
\begin{eqnarray*}
\frac{\partial f}{\partial t}=\nabla_x\cdot(\nabla_x f+f\nabla_x W)
\end{eqnarray*}
with initial condition $f(0,\cdot)=f_0$. Then, there is a constant $C$ depending on \eqref{c1}-\eqref{c2} and $s$, such that for all $t>0$,
\begin{eqnarray*}
\frac{1}{2}\left\|f(t,x)-e^{-W(x)}\right\|_{L^1}^2\leq \int_{\RR^n}f(t,x)[\log f(t,x)+W(x)]dx\leq \frac{C}{t^{\kappa}},
\end{eqnarray*}
with $\kappa=\frac{s-2}{2-\alpha}$.
\end{theorem}
If the $W(x)$ uniformly convex, i.e. $\mathrm{Hess}[W(x)] >c I_d$ for some $c>0$, the convergence will be exponential in time (see  \cite{malrieu2003convergence}). Here $W(\mathbf{y})=H(S\mathbf{y})$ in \eqref{ggs} is not uniformly convex, and we will derive conditions on $\phi$ for $W$ to satisfy Condition (1)-(3) in the above theorem. 

\begin{proposition}\label{prop_ergodicity}
For $H: \RR^{(N-1)d}\to \RR$ and $S\in \RR^{(N-1)d\times (N-1)d}$ in Eq.\eqref{hhr} and \eqref{sm}, let 
\begin{eqnarray*}
W(\mathbf{y})=H(S\mathbf{y}), \quad \forall \by \in \R^{(N-1)d}. 
\end{eqnarray*}
Then, $W$ satisfies Condition {\rm (1)-(3)} in Theorem {\rm \ref{ttv}} if the interaction potential $\Phi$ satisfies
\begin{enumerate}
\item there exists $C_0,C_1,\beta>0$ such that for all $\br\in \RR^{d(N-1)}$ {\rm(} recall that $\phi(r)=\frac{\Phi'(r)}{r}$ {\rm)}
\begin{align}
 \frac{1}{N}\left[\sum_{2\leq j\leq N}\phi(|\mathbf{r}_{1j}|)|\mathbf{r}_{1j}|^2+\sum_{2\leq i<j\leq N}\phi(|\mathbf{r}_{ij}|)|\mathbf{r}_{ij}|^2\right] \geq C_1 r^{\beta}-C_0; \label{vc2}
\end{align}
\item there exist $c>0$ and $\alpha \in (0,2)$ such that for all $x\in \RR^{d}$,  
\begin{align} \label{eq:HessPhi}
 \mathrm{Hess}_x[ \Phi(|x|) ] - c(1+|x|)^{\alpha-2}I_d \geq 0,
\end{align}
i.e., the matrix is positive definite, for all $x\in \RR^d$. 
\end{enumerate}
\end{proposition}

\begin{proof}We only need to verify Condition (2)-(3). 

For Condition (3), note that $\mathbf{r}^t=(\mathbf{r}_{12}^t,\mathbf{r}_{13}^t,\ldots,\mathbf{r}_{1N}^t)=S\mathbf{Y}^t$, when both $\mathbf{r}^t$ and $\mathbf{Y}^t$ are considered as $d(N-1)\times 1$ column vectors. Then we have
\begin{eqnarray*}
\nabla_{\mathbf{Y}}H(S\mathbf{Y})\cdot \mathbf{Y}=[\nabla_{\mathbf{r}}H(\mathbf{r})]^{T}S S^{-1}\mathbf{r}=\nabla_{\mathbf{r}}H(\mathbf{r})\cdot \mathbf{r}.
\end{eqnarray*}
Note that with $\phi(r): = \frac{\Phi'(r)}{r}$ and with $\mathbf{r}_{ik} = \mathbf{r}_{i}- \mathbf{r}_{k}$, we can write the gradient of $H(\mathbf{r})$ in Eq.\eqref{hhr} as 
\begin{align}\label{eq:gradHr}
\nabla_{\mathbf{r}_{1j}}H(\mathbf{r})= \frac{1}{N} \phi(|\mathbf{r}_{1j}|) \mathbf{r}_{1j} + \frac{1}{N}\sum_{k=2}^N \phi(|\mathbf{r}_{jk}|) \mathbf{r}_{kj} ,\quad \text{ for } j=2, \cdots, N.  
\end{align} 
Hence we have
\begin{align*}
\nabla_{\mathbf{r}} H(\mathbf{r})\cdot \mathbf{r}= \frac{1}{N}\left[\sum_{2\leq j\leq N}\phi(|\mathbf{r}_{1j}|)|\mathbf{r}_{1j}|^2+\sum_{2\leq i<j\leq N}\phi(|\mathbf{r}_{ij}|)|\mathbf{r}_{ij}|^2\right].
\end{align*}
Then Condition (3) follows from (\ref{vc2}).

To prove Condition (2), note that
\begin{eqnarray*}
\mathrm{Hess}_{\mathbf{Y} }H(S\mathbf{Y} )=S^{T}\mathrm{Hess}_{\mathbf{r} }H(\mathbf{r} )S, 
\end{eqnarray*}
and recall that $S^TS = A$ so $S^{-T}S^{-1} = A^{-1}$ with 
\begin{eqnarray*}
A^{-1}=\frac{1}{N}\left(\begin{array}{cccc} (N-1)I_d&- I_d&\ldots&- I_d\\ - I_d&(N-1)I_d&\ldots&- I_d\\ \vdots&\vdots&\ddots& \vdots\\ - I_d&- I_d&\ldots&(N-1)I_d\end{array}\right)\in \RR^{(N-1)d\times (N-1)d}.
\end{eqnarray*}
Then, to show that $\mathrm{Hess}_{\mathbf{Y} }H(S\mathbf{Y} )-c(1+|\mathbf{Y} |)^{\alpha-2}I_{(N-1)d}$ is positive semi-definite, it suffices to show that 
\begin{align}\label{matHA-1}
\mathbf{M}= \mathrm{Hess}_{\mathbf{r} }H(\mathbf{r} )-c(1+|S^{-1}\mathbf{r} |)^{\alpha-2}A^{-1} \geq 0. 
\end{align}

Continuing from \eqref{eq:gradHr}, we have  
 \begin{align*}
\nabla_{\mathbf{r}_{1i} }\nabla_{\mathbf{r}_{1j} }H(\mathbf{r} )= \frac{1}{N} \left[ \delta_{ij} B_i - \delta_{i\neq j}  A_{ij}  + \delta_{ij} \sum_{k=2, k\neq i}^N A_{ik}  \right],
\end{align*}
where $B_i$ and $A_{ij}$ are $d\times d$ matrices (recalling that $\mathbf{r}_{ik} = \mathbf{r}_{i}- \mathbf{r}_{k}$) given by
\begin{align*}
& B_i= \nabla_{\mathbf{r}_{1i} } \left[ \phi(|\mathbf{r}_{1i}|) \mathbf{r}_{1i} \right] = \phi(|\br_{1i}|) I_d + \frac{\phi'(|\br_{1i} |)}{|\br_{1i}|}\br_{1i}\otimes \br_{1i} = \mathrm{Hess}_{\mathbf{r}_{1i} }\Phi(|\mathbf{r}_{1i} |) ,\\
& A_{ik}= \nabla_{\mathbf{r}_{1i} } \left[ \phi(|\mathbf{r}_{ik}|) \mathbf{r}_{ik} \right] = \phi(|\br_{ik}|) I_d + \frac{\phi'(|\br_{ik} |)}{|\br_{ik}|}\br_{ik}\otimes \br_{ik}= \mathrm{Hess}_{\mathbf{r}_{ik} }\Phi(|\mathbf{r}_{ik} |)
\end{align*}
for $i\neq k$, and where we have used the fact that $A_{ik}= - \nabla_{\mathbf{r}_{1k} } \left[ \phi(|\mathbf{r}_{ik}|) \mathbf{r}_{ik} \right]$ to obtain the term $- \delta_{i\neq j}  A_{ij}$.
Hence, the diagonal and off-diagonal entries of the $(N-1)d\times (N-1)d$ matrix $\mathbf{M}$ in \eqref{matHA-1}  can be written as 
\begin{eqnarray*}
\mathbf{M}=\frac{1}{N}\left(\begin{array}{cccc} B_2-C +  \sum_{j\neq 2}D_{2j}&- D_{23}&\ldots&- D_{2N}\\ - D_{23}&B_3-C +  \sum_{j\neq 3}D_{3j}&\ldots&- D_{3N}\\ \vdots&\vdots&\ddots& \vdots\\ - D_{N2}&- D_{N3}&\ldots&B_N-C +  \sum_{j\neq N}D_{Nj} \end{array}\right)
\end{eqnarray*}
with $C= c(1+|S^{-1}\mathbf{r} |)^{\alpha-2}I_d$, $D_{ij}= A_{ij} - C $. 

Since $D_{ij}=D_{ji}$ for all $2\leq i<j\leq N$, for any $\eta = (\eta_2, \eta_3, \dots, \eta_{N})\in \RR^{(N-1)d} $, we have 
\begin{align}
N \eta \mathbf{M} \eta^{T} & = \sum_{I=2}^{N} \langle (B_i - C)\eta_i, \eta_i\rangle  + \sum_{2\leq i,j\leq N, i\neq j}  \langle \mathbf{D}_{ij} \eta_i, \eta_i \rangle - \langle  \mathbf{D}_{ij} \eta_i, \eta_j \rangle  \notag \\
& =  \sum_{I=2}^{N} \langle (B_i - C)\eta_i, \eta_i\rangle  + \sum_{2\leq i<j\leq N} \langle  \mathbf{D}_{ij}( \eta_i- \eta_j), (\eta_i  -\eta_j)\rangle, \label{M_positiveD}
\end{align}
where the second equality follows from the fact that $D_{ij}$ is symmetric.  Note that  the eigenvalues of $S$ are $\left\{\sqrt{\frac{k}{k-1}}\right\}_{k=2}^{N}$, so $|S^{-1}\mathbf{r} |\geq \frac{1}{\sqrt{2}} |\mathbf{r}| $. Since $\alpha\in(0,2)$, we have 
\[ (1+|S^{-1}\mathbf{r} |)^{\alpha-2} \leq \left(1+ \frac{1}{\sqrt{2}}|\mathbf{r} |\right)^{\alpha-2} \leq 2^{\frac{2-\alpha}{2}} (\sqrt{2}+|\mathbf{r}|)^{\alpha-2}. \]
Noticing that $\alpha>0$, we have 
\[C= c(1+|S^{-1}\mathbf{r} |)^{\alpha-2}I_d \leq 2c (1+|\mathbf{r}_{1i} |)^{\alpha-2}I_d
\]
for each $i\in\{2,3,\ldots,N\}$. Hence, we have,  for each $i,j\in \{2, \dots, N\}$,
\begin{align*}
 B_i-C &\geq   \mathrm{Hess}_{\mathbf{r}_{1i} }\Phi(|\mathbf{r}_{1i} |) -2 c(1+|\mathbf{r}_{1i} |)^{\alpha-2}I_d \geq 0 , \\
 D_{ij} & \geq   \mathrm{Hess}_{\mathbf{r}_{ij} }\Phi(|\mathbf{r}_{ij} |) - 2c(1+|\mathbf{r}_{1j} |)^{\alpha-2}I_d \geq 0 
 \end{align*}
Plugging them into Eq.\eqref{M_positiveD}, we obtain that $\mathbf{M}$ is positive definite. 
\end{proof}

\begin{lemma}\label{lemmaPhi} Let $\phi(r):= \frac{\Phi'(r)}{r}$. Then, 
Condition \eqref{vc2} holds if
\begin{itemize}
    \item[\rm{(i)}] there exists $c_1,c_2>0$ such that  $\phi(r)\geq 0$ and $\phi(r)r^2 \geq c_1r^\beta -c_2$ for all $r>0$; 
\end{itemize}
and Condition \eqref{eq:HessPhi}
holds if either \rm{(ii)} or \rm{(iii)} is true: 
\begin{itemize}
    \item[\rm{(ii)}]  $\phi'(r)\geq 0$ and there exists $c_3>0$ such that $\phi(r) \geq c_3 (1+r)^{\alpha-2}$ for all $r>0$; 
    \item[\rm{(iii)}] $\phi'(r)\leq 0$ and there exists $c_3>0$ such that $\phi(r) + \phi'(r)r \geq c_3 (1+r)^{\alpha-2}$ for all $r>0$. 
\end{itemize}
\end{lemma}
\begin{proof}
Suppose \rm{(i)} is true, then Condition \eqref{vc2} follows from 
\begin{align*}
 & \frac{1}{N}\left[\sum_{2\leq j\leq N}\phi(|\mathbf{r}_{1j}|)|\mathbf{r}_{1j}|^2+\sum_{2\leq i<j\leq N}\phi(|\mathbf{r}_{ij}|)|\mathbf{r}_{ij}|^2\right]\geq  \frac{1}{N}\sum_{2\leq j\leq N}\phi(|\mathbf{r}_{1j}|)|\mathbf{r}_{1j}|^2 \\
 \geq &  c_1 \sum_{2\leq j\leq N} |\mathbf{r}_{1j}|^\beta - c_2 \geq \frac{1}{N} c_1c_\beta |\br|^\beta- c_2,
\end{align*}
where $c_\beta $ is the minimum of $f(\br) := \sum_{2\leq j\leq N} |\br_{1j}|^{\beta}$ on the unit sphere $\{ \br = (\br_{12},\dots,\br_{1N}):  \sum_{2\leq j\leq N } |\br_{1j}|^2 =1 \} $. 

To prove Condition \eqref{eq:HessPhi}, note that
\begin{align*}
 \mathrm{Hess}_x \Phi(|x|) =  \nabla (\phi(|x|)x) = \phi'(|x|)\frac{x\otimes x}{|x|} + \phi(|x|)I_d .
\end{align*}
Then, if \rm{(ii)} is true, Condition \eqref{eq:HessPhi} follows directly, since $\frac{x\otimes x}{|x|}$ is positive definite. If \rm{(iii)} is true, using the fact that $I_d - \frac{x\otimes x}{|x|^2}$ is positive definite,  we obtain Condition \eqref{eq:HessPhi}: 
\begin{align*}
\mathrm{Hess}_x \Phi(|x|) &=  \left[\phi(|x|) +\phi'(|x|)|x| \right] I_d - \phi'(|x|)|x| \left( I_d -  \frac{x\otimes x}{|x|^2} \right) \\
& \geq  \left[\phi(|x|) +\phi'(|x|)|x| \right] I_d \\
& \geq c_3 (1+|x|)^{\alpha -2} I_d.
\end{align*}
\end{proof}
\begin{example} \label{exmp_gamma>1}
The function $\Phi(r) = (a+r^2)^\gamma$ with $a>0$ and $\gamma \geq 1$ satisfies Condition \eqref{vc2}- \eqref{eq:HessPhi}. To see this, note that 
\[\phi(r)= \frac{\Phi'(r)}{r} = 2\gamma(a+r^2)^{\gamma-1}\geq  2\gamma a^{\gamma-1},\] and $\phi(r)r^2 \geq c_1 r^2  -c_2$ with $c_1 = 2\gamma a^{\gamma-1}$  and $c_2=0$. Therefore, Condition \eqref{vc2} holds by Lemma \rm{\ref{lemmaPhi} (i)} with $\beta =2$.  Note also that
\[ \phi'(r)= 4\gamma(\gamma-1)(a+r^2)^{\gamma-2}r >0  \]
and $\phi(r) \geq 2\gamma a^{\gamma-1} \geq c_3  (1+r)^{\alpha-2}$ with $c_3 = 2\gamma a^{\gamma-1}$ for any $\alpha \in (0,2)$, i.e. Lemma \rm{\ref{lemmaPhi} (ii)} hods, and so does Condition \eqref{eq:HessPhi}.  
\end{example}

\begin{example} \label{exmp_1}
The function $\Phi(r) = (a+r^\theta)^\gamma$ with $a>0$, $\gamma \in (0,1]$, $\theta \in (1,2]$ and  $\theta\gamma >1$ satisfies Condition \eqref{vc2}- \eqref{eq:HessPhi}. To see this, note that $\phi(r)=\theta\gamma(a+r^\theta)^{\gamma-1}r^{\theta-2}\geq 0$ and
\begin{align*}
\phi(r)r^2 & = \theta\gamma(a+r^\theta)^{\gamma-1}r^\theta =\theta\gamma \left( \frac{a+r^\theta}{r^\theta}\right)^{\gamma-1} r^{\theta\gamma} \\
&= c_1 r^{\theta\gamma}+ [ \theta\gamma \left( \frac{r^\theta}{a+r^\theta}\right)^{1-\gamma} -c_1 ]  r^{\theta\gamma}  \geq   c_1 r^{\theta\gamma} -c_0
\end{align*}
with $c_1=\theta\gamma/2 $ and with $-c_0$ being the minimum of $f(r) = \gamma[ \theta \left( \frac{r^\theta}{a+r^\theta}\right)^{1-\gamma} - 1/2 ]  r^{\theta\gamma}$ (whose maximum exists and is positive because $f(r) < 0$ for small $r$  and $\lim_{r\to\infty} f(r) = +\infty$).  Then, Condition \eqref{vc2} holds by Lemma {\rm \ref{lemmaPhi} (i)} with $\beta =\theta\gamma$.  
For Condition \eqref{eq:HessPhi}, note that 
\[\phi'(r)= \theta\gamma(a+r^\theta)^{\gamma-2}r^{\theta-3} [\theta-2 +\theta(\gamma-1)r^{\theta} (a+r^\theta)^{\gamma-2} ] \leq 0,
\] 
because $\theta \leq 2$ and $\gamma<1$. Noting that $\theta -1+ \theta(\gamma-1)r^\theta(a+r^\theta)^{-1} \geq \theta-1 + \theta(\gamma -1) = \theta\gamma -1 >0$, we have
 \begin{align*}
\phi(r) + \phi'(r)r & = \theta\gamma(a+r^\theta)^{\gamma-1}r^{\theta-2} [\theta -1+ \theta(\gamma-1)r^\theta(a+r^\theta)^{-1}] \\
&\geq \theta\gamma (\theta\gamma -1)(a+r^\theta)^{\gamma-1} r^{\theta-2}  \\
&= \theta\gamma (\theta\gamma -1)\left(\frac{r^\theta}{a+r^\theta}\right)^{1-\gamma} \left(\frac{r}{1+r}\right)^{\theta\gamma-2} (1+r)^{\theta\gamma-2} \\
&\geq c_3  (1+r)^{\theta\gamma-2 },
 \end{align*}
 where $c_3 $ is the minimum of the function $f(r)=  \theta\gamma(\theta\gamma -1) \left(\frac{r^\theta}{a+r^\theta}\right)^{1-\gamma} \left(\frac{r}{1+r}\right)^{\theta\gamma-2}$ (whose minimum exists and is positive because $ \theta\gamma >1$ and $\theta\leq 2$).
 Thus, Lemma \rm{\ref{lemmaPhi} (iii)} hods with $\alpha =2\gamma$.  
\end{example}

\begin{example} \label{exmp_2}
The function $\Phi(r) = r^{\gamma}$ satisfies Condition \eqref{vc2}- \eqref{eq:HessPhi} only when $\gamma \in (1,2]$. To see this, note that $\phi(r) =  \gamma r^{\gamma -2}$ and $\phi(r)r^2 =\gamma r^\gamma$. Then, Lemma {\rm \ref{lemmaPhi} (i)} holds with $\beta =\gamma$, $c_1=\gamma$ and $c_2=0$. Thus, Condition \eqref{vc2} holds for any $\gamma>0$. For Condition \eqref{eq:HessPhi}, note that $
\phi'(r) = \gamma (\gamma -2) r^{\gamma -3}\leq 0$ when $\gamma\leq 2$. Also, note that 
\[
\phi(r)+ \phi'(r)r = \gamma r^{\gamma -2}(\gamma -1) \geq \gamma (\gamma -1) (1+r)^{\gamma-2}.
\]
Thus, Lemma \rm{\ref{lemmaPhi} (iii)} holds with $\alpha =\gamma$ and with $c_3=\gamma (\gamma -1)>0$ when $\gamma \in (1,2]$.  When $\gamma>2$, we have $\phi'(r)\geq 0$, but $\phi(r)/(1+r)^{\alpha-2} = \gamma r^\gamma(1+r)^{2-\alpha}$ has a minimum $0$ over $r\in (0,\inf ty)$ for any $0<\alpha<2$. Thus, Lemma \rm{\ref{lemmaPhi} (ii)} does not hold. Then, Condition \eqref{eq:HessPhi} holds only when $\gamma \in (1,2]$. 
\end{example}

\section{Coercivity for stationary processes}\label{sec:PD_inv}
We show that the coercivity conditions holds when the system of relative positions is stationary and has a potential in the form of $\Phi_0$ in \eqref{eq:Phi_0def}, or more generally, the potential
\begin{equation}\label{eq:phi_0All}
\Phi(r)= c_1\Phi_0(r)+c_2\Psi(r),
\end{equation}
where $c_1>0,c_2\geq 0$ and $\Psi:[0,\infty)\to \RR$ is a function such that $\Psi(|u-v|)$ is negative definite.

Note that when the system of relative positions  \eqref{ggs}  is stationary, the  functional $\overline{ I}_h$ in \eqref{eq:c_t1} in the coercivity condition becomes 
\begin{equation}\label{eq:I_infty}
\overline{ I}_h = I_\infty(h):=\int_{\RR^d}\int_{\RR^d} h(|u|)h(|v|)\frac{\langle u,v \rangle}{|u||v|} p_\infty(u,v)dudv
\end{equation}
with $p_\infty(u,v)$  being the stationary density in \eqref{puv}. Also, the integral kernel $K_T(u,v)$ in \eqref{eq:I_Integral_kernel} becomes $K_T(u,v)= \frac{\langle u,v \rangle}{|u||v|} p_\infty(u,v)$. Furthermore, abusing the notation $ \|h\|_{L^2(\rho)}^2$ for $ \|h(|\cdot|)\|_{L^2(\rho)}^2$, we have
\[
\int_0^T \E[h(|\br_{12}^t|)^2]dt =  \|h\|_{L^2(\rho)}^2 <\infty
\] if and only if $h(|\cdot|) \in L^2(\rho)$ with $\rho$ being the marginal density of $p_\infty(u,v)$:
 \begin{align}\label{eq:rho}
 \rho(u) = \int p_\infty(u,v) dv. 
 \end{align}
 Hence, Equation \eqref{eq:c_t1} in the coercivity condition is equivalent to 
 \begin{equation}\label{eq:CC_infty}
I_\infty(h)=\int_{\RR^d}\int_{\RR^d} h(|u|)h(|v|)\frac{\langle u,v \rangle}{|u||v|} p_\infty(u,v)dudv \geq c_{\hypspace} \|h\|_{L^2(\rho)}^2.
\end{equation}

\begin{theorem}[Coercivity] \label{thm:coercovity0}
The coercivity condition in {\rm Definition \ref{def:coercivity}} holds true on $L^2(\rho)$ for the system \eqref{ggs} with the potential $\Phi$ in \eqref{eq:phi_0All} if the system starts from the stationary density $p_\infty(u,v)$ in \eqref{puv}.   That is, for any finite dimensional linear subspace $\hypspace \subset L^2(\rho)$ with $\rho$ in \eqref{eq:rho}, there exists a constant $c_{\hypspace}>0$ such that for all $h\in \hypspace$, Equation \eqref{eq:CC_infty} holds. 
\end{theorem}

\begin{proof} 
By Theorem \ref{main} below, $I_\infty(h)\geq 0$ and the equality holds only when $h=0$ everywhere. Note also that $I_\infty(h)$ is a continuous functional on $L^2(\rho)$ since by H\"older's inequality,  
\[
I_\infty(h)\leq \|h\|_{L^2(\rho)}^2. 
\]
Thus,  \eqref{eq:CC_infty} holds true with 
\[
c_\hypspace = \min_{h\in \hypspace,\, \|h\|^2_{L^2(\rho)} =1} I_\infty(h),
\]
when $\hypspace \in L^2(\rho)$ is a finite dimensional linear subspace. 
\end{proof}

The coercivity condition holds for non-radial functions, because the above proof is for general function in $L^2(\rho)$. It is based on the following theorem: 
\begin{theorem}\label{main}
Let  $p_\infty(u,v)$ be the density function in \eqref{puv} with $\Phi$ defined as \eqref{eq:phi_0All}. 
 Let $I_\infty(h)$ be the functional defined in \eqref{eq:I_infty}. Then, 
 \[ I_\infty(h)=\int_{\RR^d}\int_{\RR^d} h(u)h(v)\frac{\langle u,v \rangle}{|u||v|} p_\infty(u,v)dudv \geq 0
\] for any $h\in L^2(\RR^d, \rho)$ with $\rho$ defined in \eqref{eq:rho}, and $I_\infty(h) = 0$ only when $h=0$ almost everywhere. 
\end{theorem}

 Theorem \ref{main} is equivalent to that the integral kernel $K_\infty(u,v)= \frac{\langle u,v \rangle}{|u||v|} p_\infty(u,v)$ in the definition of $I_\infty(h)$ is strictly positive definite. To prove this, we start with a few technical results about positive definite integral kernels.  

We introduce the following notation to simplify the expression of the marginalization in $I_\infty(h)$. For any function $\Phi:\R^+\to \R$, 
 we denote 
\begin{equation}\label{eq:triangle_box}
    \begin{aligned}
     D_\Phi^\triangle(u,v) &=\Phi(|u|)+\Phi(|v|)-\Phi(|u-v|); \\
     D_\Phi^\Box(u,v)  &=\Phi(|u|)+\Phi(|v|)-\Phi(|u-v|) -\Phi(0).
    \end{aligned}
\end{equation}
Then, for the function $\Phi$ in \eqref{eq:phi_0All}, we have 
\begin{equation}\label{eq:phi_decomp}
\begin{aligned} 
& \Phi(|u|)+\Phi(|v|)+\Phi(|u-v|) \\
=&  -\left[\Phi(|u|)+\Phi(|v|) - \Phi(|u-v|)\right] +\Phi(|u|)+\Phi(|v|)\\
=&  - c_1 D_{\Phi_0}^\triangle(u,v) - c_2 D_{\Psi}^\Box(u,v) + c_2\Psi(0) +\Phi(|u|)+\Phi(|v|). 
\end{aligned}
\end{equation}

Recall that $p_\infty(u,v)$ in \eqref{puv} is 
\begin{equation*}
p_\infty(u,v)=\frac{1}{Z}f(u,v)e^{-\frac{2}{N}\left[\Phi(|u|)+\Phi(|v|)+\Phi(|u-v|)\right]},
\end{equation*} 
where 
\begin{align*}
f(u,v)= \int e^{-\frac{2}{N}\left[\sum_{4\leq i<j}^N\Phi(|\mathbf{r}_{1i}-\mathbf{r}_{1j}|)+\sum_{l=4}^N\left[\Phi(|\mathbf{r}_{1l}|)+\Phi(|u-\mathbf{r}_{1l}|)+\Phi(|v-\mathbf{r}_{1l}|)\right]\right]}     d\mathbf{r}_{14}\ldots\mathbf{r}_{1N}.
\end{align*}
For any fixed $\mathbf{r}_{14}, \mathbf{r}_{15}, \ldots, \mathbf{r}_{1N}$, let
\begin{eqnarray}
\quad \quad
\overline{h}_{\br}(u)=h(u)e^{-\frac{4}{N}\Phi(|u|)-\frac{2}{N}\sum\limits_{4\leq l\leq N}\Phi(|u-\br_{1l}|)-\frac{1}{N}\sum\limits_{4\leq i<j\leq N}\Phi(|\br_{1i}-\br_{1j}|)-\frac{1}{N}\sum\limits_{4\leq l\leq N}\Phi(|\br_{1l}|)}. \label{hr}
\end{eqnarray}
Then, combining \eqref{eq:phi_decomp} and \eqref{hr}, we have 
\begin{align}
I_\infty(h)&=\frac{1}{Z}\int_{\RR^d}\int_{\RR^d}\int_{\RR^{(N-3)d}}\overline{h}_{\br}(u)\overline{h}_{\br}(v)\frac{\langle u,v \rangle}{|u||v|}e^{\frac{2}{N}[ \Phi(|u|)+\Phi(|v|)-\Phi(|u-v|)]}d\br_{14}\cdots d\br_{1N}dudv \notag \\
&=\frac{1}{Z}e^{\frac{2c_2}{N}\Psi(0)}\int_{\RR^{(N-1)d}}\overline{h}_{\mathbf{r}}(u)\overline{h}_{\mathbf{r}}(v)\frac{\langle u,v \rangle}{|u||v|}e^{\frac{2}{N}[c_1 D_{\Phi_0}^\triangle(u,v) + c_2 D_{\Psi}^\Box(u,v)]}d\mathbf{r}_{14}\cdots d\mathbf{r}_{1N}dudv.  \label{eq:I_triangle_box}
\end{align}
The following lemma shows that $D_{\Phi_0}^\triangle(u,v)$ and $D_{\Psi}^\Box(u,v)]$ are positive definite. 

\begin{lemma}\label{lm:triangeBox} Let $\Phi:\R^+\to\R$ be a function and consider the functions $D_\Phi^\triangle(u,v)$ and $D_\Phi^\Box(u,v)$ defined in \eqref{eq:triangle_box}. 
\begin{itemize}
    \item[(i)] If $\Phi(|u-v|)$ is negative definite, then $D_\Phi^\Box(u,v)$  is positive definite. 
    \item[(ii)] If in addition that $\Phi(0)\geq 0$, then $D_\Phi^\triangle(u,v)$ is positive definite.
\end{itemize}
 In particular, $\Phi_0$ defined in \eqref{eq:Phi_0def} is negative definite and $\Phi_0(0)\geq 0$, so $D_{\Phi_0}^\triangle(u,v)$ is positive definite. 
\end{lemma}
\begin{proof} Applying Theorem \ref{tpn} to the function $\psi(u,v) = \Phi(|u-v|)$, we obtain that both $D_\Phi^\Box(u,v)$ and $D_{\Phi}^\triangle(u,v)$ are positive definite. 

To show that $D_{\Phi_0}^\triangle(u,v)$ is positive definite, we need to show that $\Phi_0(u-v)$ is negative definite. By Theorem \ref{t54}, it suffices to show that $\psi(u,v) = a+ |u-v|^2 $ is negative definite. Note that $\psi(u,v) $ is symmetric and for any $\{c_1,c_2,\dots, c_n\}\in \R$ with $\sum_{j=1}^n c_j =0$,  
\begin{align*}
\sum_{j,k=1}^n c_j c_k \psi(u_j,u_k) & = \sum_{j,k=1}^n c_j c_k \left[a+ |u_j|^2 + |u_k|^2 - 2\innerp{u_j}{u_k}\right]  \\
& = -2  \sum_{j,k=1}^n c_j c_k  \innerp{u_j}{u_k} =  -2|\sum_{j}^n c_ju_j|^2 \leq 0,
\end{align*}
for any $\{u_1,u_2,\dots, u_n\}\in \R^d$. Thus, $\psi(u,v) = a+ |u-v|^2$ is negative definite.  
\end{proof}


The proof of Theorem \ref{main} relies on the fact that the polynomials are dense in a class of weighted $L^2(\mu)$ spaces, as in the following lemma:
\begin{lemma}\label{poly} {\rm \cite[Lemma 1.1]{bs92}} Let $\mu$ be a measure on $\mathbb{R}^d$ satisfying
\[\int e^{c|x|}d\mu(x)<\infty\]
for some $c>0$, where $|x|=\sum_{j=1}^d|x_j|$. Then the polynomials are dense in $L^2(\mu)$.
\end{lemma}

\begin{proposition}\label{prop_powerKernel}
For $\Phi_0$ in \eqref{eq:Phi_0def} with $\gamma\in (0,1)$ and $h_\br:\R^d\to \R$ defined in \eqref{hr}, let
\begin{eqnarray*}
I_\br:=\int_{\RR^d}\int_{\RR^d}h_\br(u)h_\br(v)\frac{\langle u,v\rangle}{|u||v|}\left[\Phi_0(|u|)+\Phi_0(|v|)-\Phi_0(|u-v|)\right]dudv. 
\end{eqnarray*}
Then, $I_\br\geq 0$ and $I_\br=0$ if and only if $h=0$ almost everywhere. 
\end{proposition}
 \begin{proof}  From the expression (\ref{eq:Phi_0def}), we obtain
 \begin{align*}
 I_\br & =I_1+I_2 \text{ with } \\
 I_1:&=\int_{\RR^d}\int_{\RR^d}h_\br(u)h_\br(v)\frac{\langle u,v\rangle}{|u||v|}\left[(a+|u|^\theta)^{\gamma}+(a+|v|^\theta)^{\gamma}-(2a+|u|^\theta+|v|^\theta)^{\gamma}\right]dudv; \notag\\
 I_2:&=\int_{\RR^d}\int_{\RR^d}h_\br(u)h_\br(v)\frac{\langle u,v\rangle}{|u||v|}\left[(2a+|u|^\theta+|v|^\theta)^{\gamma}-(a+|u-v|^\theta)^{\gamma}\right]dudv.
 \end{align*}
 By Lemma \ref{l22} and Lemma \ref{l23}, $(a+|u|^\theta)^{\gamma}+(a+|v|^\theta)^{\gamma}-(2a+|u|^\theta+|v|^\theta)^{\gamma}$ is positive definite as a function of $(u,v)$, therefore $I_1\geq 0$.
 
 Note also that for $0<\gamma<1$,
\begin{eqnarray*}
z^{\gamma}=\frac{\gamma}{\Gamma(1-\gamma)}\int_0^{\infty}(1-e^{-\lambda z})\frac{d\lambda}{\lambda^{\gamma+1}},
\end{eqnarray*}
where $\Gamma(1-\gamma)=\int_0^{\infty}x^{-\gamma}e^{-x}dx$ is the Gamma function.
Then, with $C_\gamma:=\frac{\gamma}{\Gamma(1-\gamma)}$, and with the notation 
\begin{equation}\label{D_psi}
D^\triangle_\psi(u,v) = |u|^\theta + |v|^\theta - |u-v|^\theta
\end{equation} 
for the function $\psi(u)=|u|^\theta$ as in \eqref{eq:triangle_box}, we have
\begin{align}
I_2&=C_\gamma \int_0^{\infty}\int_{\RR^d}\int_{\RR^d}h_\br(u) h_\br(v)\frac{\langle u,v \rangle}{|u||v|}\left[e^{-\lambda(a+|u-v|^\theta)}-e^{-\lambda(2a+|u|^\theta+|v|^\theta)}\right] dudv\frac{d\lambda}{\lambda^{\gamma+1}} \notag \\
& =C_\gamma\int_0^{\infty}\int_{\RR^d}\int_{\RR^d}h_\br(u)h_\br(v)\frac{\langle u,v\rangle}{|u||v|}e^{-\lambda(|u|^\theta+|v|^\theta+a)}[e^{2\lambda D^\triangle_\psi(u,v)}-e^{-\lambda a}]dudv\frac{d\lambda}{\lambda^{\gamma+1}} \notag \\
& =C_\gamma\int_0^{\infty}\int_{\RR^d}\int_{\RR^d}h_\br(u)h_\br(v)\frac{\langle u,v\rangle}{|u||v|}e^{-\lambda(|u|^\theta+|v|^\theta+a)}e^{2\lambda D^\triangle_\psi(u,v)}dudv\frac{d\lambda}{\lambda^{\gamma+1}}, \label{eq:I2_theta}
\end{align}
where the last equality follows from the fact that 
\[
\int_{\RR^d}\int_{\RR^d}h_\br(u)h_\br(v)\frac{\langle u,v\rangle}{|u||v|}e^{-\lambda(|u|^2+|v|^2+a)}e^{-\lambda a}dudv = 0
\]
due to symmetry. 

Consider first the case when $\theta=2$. Note that $ D^\triangle_\psi(u,v) = \langle u,v\rangle$. We have, by the Taylor expansion of $e^{2\lambda D^\triangle_\psi(u,v)}$, 
\begin{align}
I_2 &= C_\gamma\sum_{k=0}^{\infty}\int_0^{\infty}e^{-\lambda a}\int_{\RR^d}\int_{\RR^d} \left[\frac{h_\br(u)e^{-\lambda|u|^2}}{|u|}\right]\left[\frac{h_\br(v)e^{-\lambda |v|^2}}{|v|}\right]\frac{(2\lambda)^k\langle u,v\rangle^{k+1}}{k!} dudv\frac{d\lambda}{\lambda^{\gamma+1}}\notag\\
&= C_\gamma\sum_{\substack{i_1,\ldots,i_d\geq 0,\\ k=i_1+\ldots +i_d\geq 1}}\int_0^{\infty} e^{-\lambda a}\frac{(2\lambda)^{k-1}k}{i_1!\cdot \ldots \cdot i_d!}\left|\int_{\RR^d} \frac{h_\br(u)e^{-\lambda|u|^2}u_1^{i_1}\cdot\ldots u_d^{i_d}}{|u|} du\right|^2\frac{d\lambda}{\lambda^{\gamma+1}}.  \label{ig}
\end{align}
Hence we have $I_\br\geq 0$, and $I_\br=0$ only when for all the $i_1,\ldots, i_d\geq 0$, $i_1+\ldots+i_d\geq 1$, and almost every $\lambda\geq 0$,
\begin{eqnarray*}
\int_{\RR^d} \frac{h_\br(u)e^{-\lambda|u|^2}u_1^{i_1}\cdot\ldots u_d^{i_d}}{|u|}du=0.
\end{eqnarray*}
Then, by Lemma \ref{poly}, $I_\br=0$ if and only if $h=0$ almost everywhere in $\RR^d$.

Next, consider the case when $\theta \in (1,2)$. Note first that $D^\triangle_\psi(u,v)$ with $\psi(u) =|u|^\theta$ is positive definite by Lemma \ref{lm:triangeBox}, so is its powers. Then, continuing from \eqref{eq:I2_theta}, we have 
\begin{align*}
I_2& \geq C_\gamma\int_0^{\infty}\int_{\RR^d}\int_{\RR^d}h_\br(u)h_\br(v)\frac{\langle u,v\rangle}{|u||v|}e^{-\lambda(|u|^\theta+|v|^\theta+a)}D^\triangle_\psi(u,v) dudv\frac{d\lambda}{\lambda^{\gamma+1}} \\
& = C_\gamma e^{-\lambda a} \int_0^{\infty}\int_{\RR^d}\int_{\RR^d}\widetilde h_\br(u)\widetilde h_\br(v)\frac{\langle u,v\rangle}{|u||v|} ( |u|^{2 \gamma} + |v|^{2 \gamma} - |u-v|^{2 \gamma}) dudv\frac{d\lambda}{\lambda^{\gamma+1}},
\end{align*}
where $\widetilde h_\br (u) := h_\br (u)  e^{-\lambda |u|^\theta}$ and in the equality, we rewrite $D_\psi^\triangle (u,v) = |u|^{2 \gamma} + |v|^{2\gamma} - |u-v|^{2\gamma} $ with $\gamma = \theta/2$. This returns to the above case when $\theta=2$.   
\end{proof}

\begin{proof}[\textbf{Proof of Theorem \ref{main}} ] Note first that  $D_{\Phi_0}^\triangle(u,v)$ and $D_{\Psi}^\Box(u,v)]$ are positive definite by Lemma \ref{lm:triangeBox}. 
We prove the theorem by separating it into the following cases:
\begin{itemize}
    \item Case I: $\gamma=1$ and $\theta\in (0,2]$; 
    \item Case II:  $\gamma\in(0,1)$ and $\theta\in (1,2]$
\end{itemize}

\textbf{Case I:  when $\gamma=1$ and $\theta=(1,2]$.} We detail only the case $\theta=2$, and the proof for case when $\theta\in (1,2)$ is similar to those in Proposition \ref{prop_powerKernel} by using Gamma function. Note that $D_{\Phi_0}^\triangle(u,v) = 2\innerp{u}{v}$. With $Z_1=\frac{1}{Z}e^{\frac{2c_2}{N}\Psi(0)+\frac{2c_1 a}{N}}$, we have 
\begin{eqnarray*}
I_\infty(h)&=&Z_1 \int_{\RR^d}\int_{\RR^d}\int_{\RR^{(N-3)d}}\overline{h}_{\mathbf{r}}(u)\overline{h}_{\mathbf{r}}(v)\frac{\langle u,v \rangle}{|u||v|}e^{\frac{4c_1\langle u,v\rangle}{N}} e^{\frac{2c_2}{N}D_{\Psi}^\Box(u,v) }d\mathbf{r}_{14}\cdots d\mathbf{r}_{1N}dudv.
\end{eqnarray*}
Since $D_{\Psi}^\Box(u,v) $ is positive definite, and so is any power of it by Theorem \ref{t52}(2). Thus,
\begin{eqnarray*}
I_\infty(h)&=&Z_1\int_{\RR^d}\int_{\RR^d}\int_{\RR^{(N-3)d}}\overline{h}_{\mathbf{r}}(u)\overline{h}_{\mathbf{r}}(v)\frac{\langle u,v \rangle}{|u||v|}e^{\frac{4c_1\langle u,v\rangle}{N}}\\
&& \times \sum_{k=0}^{\infty}\frac{1}{k!}\left[\frac{2c_2}{N}D_{\Psi}^\Box(u,v) \right]^k d\mathbf{r}_{14}\cdots d\mathbf{r}_{1N}dudv \\
&\geq&Z_1\int_{\RR^d}\int_{\RR^{(N-3)d}}\overline{h}_{\mathbf{r}}(u)\overline{h}_{\mathbf{r}}(v)\frac{\langle u,v \rangle}{|u||v|}e^{\frac{4c_1\langle u,v\rangle}{N}}d\mathbf{r}_{14}\cdots d\mathbf{r}_{1N}dudv. 
\end{eqnarray*}
Expanding the term $e^{\frac{4c_1\langle u,v\rangle}{N}}$ into polynomials, we have 
\begin{align}
I_\infty(h)&\geq Z_1\int_{\RR^d}\int_{\RR^d}\int_{\RR^{(N-3)d}}\overline{h}_{\mathbf{r}}(u)\overline{h}_{\mathbf{r}}(v)\frac{1}{|u||v|}\sum_{k=0}^{\infty}\frac{(4c_1)^k\langle u,v \rangle^{k+1}}{k!N^k}d\mathbf{r}_{14}\cdots d\mathbf{r}_{1N}dudv\notag \\
&=Z_1\int_{\RR^{(N-3)d}}\sum_{\substack{i_1,\ldots,i_d\geq 0,\\k= i_1+\ldots+i_d\geq 1}}\frac{(4c_1/N)^{k-1} k}{i_1!\ldots i_d!} \left|\int_{\RR^d}\frac{\overline{h}_{\mathbf{r}}(u)u_1^{i_1}\cdots u_d^{i_d}}{|u|}du\right|^2d\mathbf{r}_{14}\cdots d\mathbf{r}_{1N}.  \label{eq:h*1}
\end{align}
Hence $I_\infty(h)\geq 0$ and $I_\infty(h)=0$ if and only if for any $i_1+\ldots+i_d\geq 1$, $i_1,\ldots,i_d\geq 0$, and almost every $(\mathbf{r}_{14},\ldots \mathbf{r}_{1N})\in\RR^{(N-3)d}$
\begin{eqnarray*}
\int_{\RR^d}\frac{\overline{h}_{\mathbf{r}}(u)u_1^{i_1}\cdots u_d^{i_d}}{|u|}du=0. 
\end{eqnarray*}
From the expression (\ref{hr}) and Lemma \ref{poly}, 
we obtain that $I_\infty(h)=0$ if and only if $h=0$ almost everywhere.

\textbf{Case II:  when $\gamma\in(0,1)$ and $\theta\in (1,2]$.}
Noticing  that 
\begin{eqnarray*}
\Phi(|u|)+\Phi(|v|)-\Phi(|u-v|)-c_2\Psi(0) = c_1 D_{\Phi_0}^\triangle(u,v) + c_2 D_{\Psi}^\Box(u,v) 
\end{eqnarray*}
is positive definite and so are its powers, we have
\begin{eqnarray*}
I& \geq &e^{\frac{2c_2}{N}\Psi(0)}\int_{\RR^d}\int_{\RR^d}\int_{\RR^{(N-3)d}}\overline{h}_{\mathbf{r}}(u)\overline{h}_{\mathbf{r}}(v)\frac{\langle u,v \rangle}{|u||v|}\\
&&\frac{2}{N}\left[ \Phi(|u|)+\Phi(|v|)-\Phi(|u-v|)-c_2\Psi(0)\right]d\mathbf{r}_{14}\cdots d\mathbf{r}_{1N}dudv\\
&=& e^{\frac{2c_2}{N}\Psi(0)}(I_3+I_4),
\end{eqnarray*}
where
\begin{eqnarray*}
I_3&=&\int_{\RR^d}\int_{\RR^d}\int_{\RR^{(N-3)d}}\overline{h}_{\mathbf{r}}(u)\overline{h}_{\mathbf{r}}(v)\frac{\langle u,v \rangle}{|u||v|} \frac{2c_1}{N}D_{\Phi_0}^\triangle(u,v) d\mathbf{r}_{14}\cdots d\mathbf{r}_{1N}dudv,\\
I_4&=&\int_{\RR^d}\int_{\RR^d}\int_{\RR^{(N-3)d}}\overline{h}_{\mathbf{r}}(u)\overline{h}_{\mathbf{r}}(v)\frac{\langle u,v \rangle}{|u||v|} \frac{2c_2}{N} D_{\Psi}^\Box(u,v)d\mathbf{r}_{14}\cdots d\mathbf{r}_{1N}dudv.
\end{eqnarray*}
Note that $I_4\geq 0$ by the positive definiteness of $D_{\Psi}^\Box(u,v)$. We may write
\begin{eqnarray*}
I_3&=&e^{\frac{2c_2}{N}\Psi(0)}\int_{\RR^{(N-3)d}}I_{\mathbf{r}}d\mathbf{r}_{14}\cdots d\mathbf{r}_{1N},
\end{eqnarray*}
where
\begin{eqnarray*}
I_{\mathbf{r}}&=&\int_{\RR^d}\int_{\RR^d}\overline{h}_{\mathbf{r}}(u)\overline{h}_{\mathbf{r}}(v)\frac{\langle u,v \rangle}{|u||v|}\frac{2c_1}{N}D_{\Phi_0}^\triangle(u,v) dudv
\end{eqnarray*}
By Proposition \ref{prop_powerKernel}, we obtain that $I_{\mathbf{r}}\geq 0$, and $I_{\mathbf{r}}=0$ if and only if $h_{\mathbf{r}}=0$ almost everywhere on $\RR^d$. Hence we have $I_3\geq 0$, and $I_3>0$ for all $h\neq 0$. Then the theorem follows.
\end{proof}

 The major effort in the above proof of Theorem \ref{main} is to deal with the inner product term $\frac{\langle u,v \rangle}{|u||v|}$ in the definition of $I_\infty(h)$. When the inner product is removed, the above proof directly implies that $p_\infty(u,v)$ is strictly positive definite. The following lemma shows that the function $f(u,v)$ in \eqref{fuv}, which is part of $p_\infty(u,v)$,  is also strictly positive definite. 
\begin{lemma} \label{lm:fuv_pd}
Let $\Phi:\RR^d\to \RR$ be a continuous function. 
\begin{enumerate}
\item  For any fixed $\mathbf{r}_{14},\mathbf{r}_{15},\ldots,\mathbf{r}_{1N} \in \RR^d$,
\begin{eqnarray*}
g_{\mathbf{r}}(u,v):=\sum_{4\leq i<j\leq N}\Phi(|\mathbf{r}_{1i}-\mathbf{r}_{1j}|)+\sum_{1\leq l\leq N}\left[\Phi(|\mathbf{r}_{1l}|)+\Phi(|u-\mathbf{r}_{1l}|)+\Phi(|v-\mathbf{r}_{1l}|)\right],
\end{eqnarray*}
as a function of $(u,v)$ is negative definite.
\item The function $f(u,v):\RR^d\times \RR^d\rightarrow \RR$, as defined in \eqref{fuv}, is strictly positive definite.
\end{enumerate}
\end{lemma}

\begin{proof}We first prove Part (1). First note that $g_{\mathbf{r}}(u,v)=g_{\mathbf{r}}(v,u)$. Moreover, for any real numbers $\{c_i\}_{i=1}^k$ satisfying $\sum_{i=1}^k c_i=0$, we have
$ \sum_{i=1}^k\sum_{j=1}^{k} c_i c_j g_{\mathbf{r}}(u_i,u_j)=0$.
Then, following from Definition \ref{def_spd}, the function $g_{\mathbf{r}}(u,v)$ is negative definite.

The strict positive definiteness of $f(u,v)$ follows from the fact that for any fixed $(\mathbf{r}_{14},\mathbf{r}_{15},\ldots,\mathbf{r}_{1N})$, the integrand $e^{-\frac{2}{N}g_{\mathbf{r}}(u,v)}$ is strictly positive definite as a function of $(u,v)$ since by Part (1), $g_{\mathbf{r}}(u,v)$ is negative definite. 
\end{proof}

\section{Coercivity for non-stationary processes}\label{sec:coercivityMean}

We show in this section that the coercivity condition holds true for the system of relative positions starting from a non-stationary distribution, provided that $T$ is large enough. 

We start by introducing the following integrals for the expectations in \eqref{eq:c_t1}.  Let $p_t(u,v)$ be the density function for $(\br_{12}^t,\br_{13}^t)$, and correspondingly, with $p_{\infty}(u,v)$ being the stationary density in  \eqref{puv}.  Let 
\begin{equation}\label{eq:I_t}
I_t(h):=\int_{\RR^d}\int_{\RR^d} h(|u|)h(|v|)\frac{\langle u,v \rangle}{|u||v|} p_t(u,v)dudv
\end{equation}
for $t\in [0,\infty]$. Then, Equation \eqref{eq:c_t1} in the coercivity condition can be written as  
\begin{equation}\label{eq:I_T_bar}
\begin{aligned}
\overline{ I}_T(h) & = \frac{1}{T} \int_0^T I_t(h)dt = \int_{\RR^d}\int_{\RR^d} h(|u|)h(|v|)\frac{\langle u,v \rangle}{|u||v|} \widebar p_T(u,v)dudv  \\
&\geq c_{\hypspace,T}\int_0^T \E[h(|\br_{12}^t|)^2]dt =  \|h\|_{L^2(\rho_T)}^2,
\end{aligned}
\end{equation}
where $\widebar p_T$ and $\widebar \rho_T$ are defined by 
\begin{equation}\label{eq:meanPDF}
\widebar p_T(u,v) = \frac{1}{T}\int_0^T p_t(u,v) dt, \quad
\widebar \rho_T (u) = \int_{\R^d} \widebar p_T(u,v) dv. 
\end{equation}
Here we abuse the notation $ \|h\|_{L^2(\rho)}^2$ for $ \|h(|\cdot|)\|_{L^2(\rho)}^2$. 

\begin{theorem}[Coercivity] \label{thm:coercivityMean}
Consider the system \eqref{ggs} with  
 the potential $\Phi$ given by \eqref{eq:phi_0All} and with initial condition satisfying \eqref{eq:IC}. Let $\widebar{\rho}_T$ and $I_\infty(h)$ be defined as in \eqref{eq:meanPDF} and \eqref{eq:I_infty}, respectively.   
 Then, the coercivity condition holds true on any finite-dimensional linear space $\mathcal{H}\subset L^2(\widebar\rho_T)$ s.t.
\begin{equation}\label{S_H}
S_\hypspace:= \sup_{0\neq h\in \mathcal{H}}  \frac{\|h\|_\infty^2}{I_\infty(h)} <\infty,
\end{equation}
when $T>(1+ 4S_\hypspace)T_{c,\hypspace}$ with $T_{c,\hypspace}:=(2C S_\hypspace)^{2/\kappa}$, where $\kappa$ and $C$ are given in Proposition \ref{prop:Conv_poly} and $\|h\|_\infty= \sup_{u\in\RR^d}|h(u)|$. 
That is, there exists $c_{\hypspace, T}>0$ such that Equation \eqref{eq:I_infty} holds for all $h\in \hypspace$. 
\end{theorem}

\begin{remark} (i) Condition \eqref{S_H} is in fact a strong version of the coercivity condition with respect to the uniform norm at the stationary density: it is equivalent to 
\begin{equation*}
    I_\infty(h) 
    \geq S_\hypspace^{-1} \|h\|_\infty^2, \text{ for all } h\in \hypspace.  
\end{equation*}
By requiring more regularity (uniform norm versus $L^2(\rho)$) at the stationary density, we gain coercivity in time. (ii) An analytical verification of Condition \eqref{S_H} is often out of reach except the Gaussian case, since the stationary measure $p_\infty$ involves a marginalization of a non-Gaussian distribution.  One may obtain a glimpse of the condition on $h$ from \eqref{eq:h*1} and \eqref{ig}, and may require that 
\begin{equation*}\label{S_H*}
S_\hypspace:= \sup_{0\neq h\in \mathcal{H}}  \frac{\|h\|_\infty^2}{\|h\|_*} <\infty,
\end{equation*}
using that $I_\infty(h)$ may be bounded from below by $\|h\|_*^2$ given by
\[\|h\|^2_* := 
\sum_{\substack{i_1,\ldots,i_d\geq 0,\\ k=i_1+\ldots +i_d\geq 1}}\int_0^{\infty} e^{-\lambda a}\frac{(2\lambda)^{k-1}k}{i_1!\cdot \ldots \cdot i_d!} \left|\int_{\RR^d} \frac{h(u)e^{-\lambda|u|^2}u_1^{i_1}\cdot\ldots u_d^{i_d}}{|u|} du\right|^2\frac{d\lambda}{\lambda^{\gamma+1}}. 
\]
(iii) On the other hand, in computational practice, this condition can be easily verified because 
\[I_\infty(h) = \E_\infty\left[h(\br_{12})h(\br_{13})\frac{\innerp{\br_{12}}{\br_{13}}}{|\br_{12}||\br_{13}|}\right] \]
and the expectation can be approximated by samples from data. 
\end{remark}

The proof of Theorem \ref{thm:coercivityMean} is based on the following lemma, which shows that $I_t(h)$ is close to $I_\infty(h)$ when $t$ is large, hence $\overline I_T(h)$ can be bounded below by a factor of $I_\infty(h)$ when $T$ is large. 
\begin{lemma}\label{lemma5.3}
 For each nonzero $h\in L^2( \widebar\rho) $, recall $\overline I_T(h)$ and $I_t(h)$ in \eqref{eq:I_t} and \eqref{eq:I_T_bar}. Let 
 \begin{equation}\label{eq:Tch}
 T_{c,h}: = \left(\frac{2C\|h\|_\infty^2}{I_\infty(h)}\right)^{\frac{2}{\kappa}} <\infty, 
 \end{equation} 
 where $\kappa$ and $C$ are given in Proposition {\rm\ref{prop:Conv_poly}}. 
 We have 
 \begin{itemize}
 \item[(1)]   $I_t(h)\geq \frac{1}{2}I_\infty(h)$ if $ t>  T_{c,h}$; 
\item[(2)] $\overline I_T(h) \geq \frac{\|h\|_\infty^2}{I_\infty(h)+4\|h\|_\infty^2}I_\infty(h)$ if 
 $ T\geq \left(1+\frac{4\|h\|_\infty^2}{I_\infty(h)}\right)T_{c,h}$. 
\end{itemize}
\end{lemma}

\begin{proof}We first prove Part (1). Note that Proposition \ref{prop:Conv_poly} implies $\| p_t -p_\infty\|_\infty\leq Ct^{-\kappa/2}$. Then, 
\begin{align*}
|I_t(h) - I_\infty(h)| = &\left|\int_{\RR^d}\int_{\RR^d}h(u)\frac{\langle u,v \rangle}{|u||v|} (p_t(u,v)-p_{\infty}(u,v))h(v)dudv\right| \\
\leq & \|h\|_\infty^2\int_{\RR^d}\int_{\RR^d}\left|p_t(u,v)-p_{\infty}(u,v)\right|dudv \leq \|h\|_\infty^2 Ct^{-\kappa/2}. 
\end{align*}
Hence,
$ I_t(h) \geq I_\infty(h)- \left| I_\infty(h) - I_t(h) \right| \geq  I_\infty(h)- \|h\|_\infty^2 Ct^{-\kappa/2}. $ 
Meanwhile, note that $t> T_{c,h}: = \left(\frac{2C\|h\|_\infty^2}{I_\infty(h)}\right)^{\frac{2}{\kappa}}$ 
 implies $ I_\infty(h)\geq 2\|h\|_\infty^2  Ct^{-\kappa/2}$. 
Then Part (1) follows.

Now we prove Part (2). Note that by Part (1), 
\begin{eqnarray*}
\widebar I_T(h)&=& \frac{1}{T}\left[\int_0^{T_{c,h}}\int_{\RR^d}h(u)\frac{\langle u,v \rangle}{|u||v|}p_t(u,v)h(v)dudv +\int_{T_{c,h}}^{T}\int_{\RR^d}h(u)\frac{\langle u,v \rangle}{|u||v|}p_t(u,v)h(v)dudv\right] \notag \\
&\geq &\frac{1}{T}\left[\frac{1}{2}I_\infty(h)(T-T_{c,h})-\|h\|_\infty^2 T_{c,h}\right]. 
\end{eqnarray*} 
Note that $T\geq \left(1+\frac{4\|h\|_\infty^2}{I_\infty(h)}\right)T_{c,h}$ implies both $ \frac{1}{2}I_\infty(h)(T-T_c)\geq 2\|h\|_\infty^2 T_{c,h}$ and $T-T_{c,h}\geq \frac{4\|h\|_\infty^2}{I_\infty(h)+4\|h\|_\infty^2}T$. Then, 
 \begin{align}\label{eq:I_Tch}
  \widebar I_T(h) \geq I_\infty(h)\frac{1}{4T}(T-T_{c,h})\geq \frac{\|h\|_\infty^2}{I_\infty(h)+4\|h\|_\infty^2}I_\infty(h)
  \end{align}
and Part (2) follows.
\end{proof}

\begin{proof}[{\bf Proof of Theorem \ref{thm:coercivityMean}}] We show first that $\widebar I_T(h)>0$ when $h\neq 0$ in $L^2(\widebar\rho_T)$.  
Note first that $ T_{c,\hypspace} = (2C S_\hypspace)^{2/\kappa} \geq T_{c,h}$ with $T_{c,h}$ defined in \eqref{eq:Tch}. Then,  $T>(1+ 4S_\hypspace)T_{c,\hypspace} \geq  \left(1+\frac{4\|h\|_\infty^2}{I_\infty(h)}\right)T_{c,h}$, and we can apply the first inequality in \eqref{eq:I_Tch} in the proof of Lemma \ref{lemma5.3} to obtain 
 \begin{align*} 
 \widebar I_T(h) & \geq I_\infty(h)\frac{1}{4T}(T-T_{c,h})
  \geq  I_\infty(h) \frac{1}{4T}(T-T_{c,\hypspace}) \geq  I_\infty(h) \frac{S_\hypspace}{1+4S_\hypspace} , 
 \end{align*}
 where the last inequality follows from $T>(1+ 4S_\hypspace)T_{c,\hypspace}$. 
By Theorem \ref{main}, $I_\infty(h)=0$ only if $h=0$ almost everywhere. Thus,  $\widebar I_T(h)>0$ when $h\neq 0$ in $L^2(\widebar\rho_T)$.

Next, to prove the coercivity condition, suppose on the contrary it does not hold, that is, there exists a sequence of nonzero functions  $\{h_n\}\subset \hypspace$ such that $\frac{\widebar I_T(h_n)}{\|h_n\|_{L^2(\rho)}^2} \to 0$. Since $\hypspace $ is finite dimensional, the normalized sequence $\widebar h_n = \frac{h_n}{\|h_n\|_{L^2(\rho_T)} }$ has a convergence subsequence. Denote the limit by $\widebar h$. Then, we have $\widebar I_T(\widebar h)=0$ and $\|\widebar h \|_{L^2(\rho_T)} =1$, a contradiction. 
\end{proof}

\appendix
\section{Positive definite kernels}\label{sec:append}

In this section, we review the definitions of positive and negative definite kernels, as well as their basic properties. The following definition is a real version of the definition in \cite[p.67]{BCR84}.

\begin{definition}\label{def_spd}
Let $X$ be a nonempty set. A function $\phi: X\times X\rightarrow \RR$ is called a (real) positive definite kernel if and only if it is symmetric (i.e. $\phi(x,y)=\phi(y,x)$) and
\begin{eqnarray}
\sum_{j,k=1}^{n}c_jc_k\phi(x_j,x_k)\geq 0\label{pr}
\end{eqnarray}
for all $n\in \NN$, $\{x_1,\ldots,x_n\}\subset X$ and $\mathbf{c}=(c_1,\ldots,c_n) \in \RR^n$. The function $\phi$ is called strictly positive definite if the equality holds only when $\mathbf{c}=\mathbf{0} \in \RR^n$. We call the function $\phi$ a (real) negative definite kernel if and only if it is symmetric and
\begin{eqnarray}
\sum_{j,k=1}^{n}c_jc_k \phi(x_j,x_k)\leq 0\label{nr}
\end{eqnarray}
for all $n\geq 2$, $\{x_1,\ldots,x_n\}\in X$ and $\{c_1,\ldots,c_n\}\in \RR$ with $\sum_{j=1}^{n}c_j=0.$
\end{definition}

\noindent{\textbf{Remark.}} In the definition of positive definiteness in \cite[p.67]{BCR84}, a function $\phi: X\times X\rightarrow \CC$ is positive definite if and only if 
$\sum_{j,k=1}^{n}c_j\overline{c}_k\phi(x_j,x_k)\geq 0$ 
for all $n\in \NN$, $\{x_1,\ldots,x_n\}\in X$ and $\{c_1,\ldots,c_n\}\in \CC$, where $\overline{c}$ denotes the complex conjugate of a complex number $c$. It is straightforward to check that when $\phi$ is real-valued and symmetric, this definition is equivalent to the definition (\ref{pr}).

Similarly, In the definition of negative definiteness in \cite[p.67]{BCR84}, a function $\phi: X\times X\rightarrow \CC$ is negative definite if and only if it is Hermitian (i.e. $\phi(x,y)=\overline{\phi(y,x)}$) and 
$\sum_{j,k=1}^{n}c_j\overline{c}_k\phi(x_j,x_k)\leq 0$ 
for all $n\geq 2$, $\{x_1,\ldots,x_n\}\in X$ and $\{c_1,\ldots,c_n\}\in \CC$ with $\sum_{j=1}^{n}c_j=0$. We can again check that when $\phi$ is real-valued, this definition is equivalent to the definition (\ref{nr}). In this paper, we only consider real-valued, symmetric kernels.

 \begin{theorem}[Properties of positive definite kernels]\label{t52}
 Suppose that $k, k_1, k_2: \mathcal{X}\times\mathcal{X}\subset\mathbb{R}^d\times\mathbb{R}^d\to \mathbb{R}$ are positive definite kernels. Then
\begin{enumerate} \setlength\itemsep{0mm} 
\item $c_1k_1+c_2k_2$ is positive definite, for $c_1,c_2\ge0$

\item $k_1k_2$ is positive definite. (\cite[p.69]{BCR84})

\item $\exp(k)$ is positive definite. (\cite[p.70]{BCR84})

\item $k(f(u),f(v))$ is positive definite for any  map $f:\mathbb{R}^d\to \mathbb{R}^d$

\item Inner product $\langle u,v\rangle=\sum_{j=1}^du_jv_j$ is positive definite (\cite[p.73]{BCR84})

\item $f(u)f(v)$ is positive definite for any function $f:\mathcal{X}\to \mathbb{R}$ (\cite[p.69]{BCR84}).

\item If $k(u,v)$ is measurable and integrable, then $\iint k(u,v)dudv\ge0$ (\cite[p.524]{RKSF})
\end{enumerate}
\end{theorem}

\begin{theorem} \rm{ \cite[Theorem 3.1.17]{BCR84}} \label{pdm}
Let $\phi: X\times X\rightarrow \RR$ be symmetric. Then $\phi$ is positive definite if and only if
$\det(\phi(x_j,x_k)_{j,k\leq n})\geq 0
$
for all $n\in \NN$ and all $\{x_1,\ldots,x_n\}\subseteq X$.
\end{theorem}

\begin{theorem}\rm{ \cite[Lemma 3.2.1]{BCR84} } \label{tpn}
Let $X$ be a nonempty set, $x_0\in X$ and let $\psi: X\times X\rightarrow \RR$ be a symmetric kernel. Put $\phi(x,y):=\psi(x,x_0)+\psi(y,x_0)-\psi(x,y)-\psi(x_0,x_0)$. Then $\phi$ is positive definite if and only if $\psi$ is negative definite. If $\psi(x_0,x_0)\geq 0$, then $\phi_0(x,y)=\psi(x,x_0)+\psi(y,x_0)-\psi(x,y)$ is positive definite if and only if $\psi$ is negative definite.
\end{theorem}

\begin{theorem}\label{t53}Let $X$ be a nonempty set and let $\psi: X\times X\rightarrow \mathbb{R}$ be a kernel. Then $\psi$ is negative definite if and only if $\exp(-t\psi)$ is positive definite for all $t>0$.
\end{theorem}
\begin{proof}The complex version of this theorem is proved in Theorem 3.2.2 of  of \cite{BCR84}. The real version can be proved in a similar way.
\end{proof}

\begin{theorem}\label{t54}If $\psi:X\times X\rightarrow\RR$ is negative definite and $\psi(x,x)\geq 0$, then so are $\psi^{\alpha}$ for $0<\alpha<1$ and $\log(1+\psi)$.
\end{theorem}

\begin{proof}The complex version of this theorem is proved in Theorem 3.2.10 of \cite{BCR84}. The real version can be proved in a similar way.
\end{proof}

\begin{theorem}\rm{\cite[Proposition 3.3.2]{BCR84} }\label{t55}Let $X$ be nonempty and $\psi: X\times X\rightarrow \CC$ be negative definite. Assume $\{(x,y)\in X\times X, \psi(x,y)=0\}=\{(x,x): x\in X\}$, then $\sqrt{\psi}$ is a metric on X.
\end{theorem}

\begin{lemma}\label{l22}Let
\begin{eqnarray}
\phi(x,y)=x^{\gamma}+y^{\gamma}-(x+y)^{\gamma}.\label{pxy}
\end{eqnarray}
When $\gamma\in(0,1)$, $\phi(x,y)$ is positive definite on $[0,\infty)\times[0,\infty)$.
\end{lemma}
\begin{proof}Let $
\psi_{\gamma}(x,y)=(x+y)^{\gamma}.
$
Then $\psi_1(x,y)=x+y$ is negative definite. By Theorem \ref{t54}, for $\gamma\in (0,1)$, $\psi_{\gamma}(x,y)$ is negative definite since $\psi_{\gamma}(x,x)\geq 0$, when $x\geq 0$. Then
\begin{eqnarray*}
\phi(x,y)=\psi_{\gamma}(x,0)+\psi_{\gamma}(0,y)-\psi_{\gamma}(x,y)
\end{eqnarray*}
is positive definite by Theorem \ref{tpn}.
\end{proof}

\begin{lemma}\label{l23}Let $\phi(x,y)$ be defined as in \eqref{pxy}. Let $g: \RR^d\rightarrow[0,\infty)$ be a function. Then 
\begin{eqnarray*}
\tilde{\phi}(u,v)=\phi(g(u),g(v))
\end{eqnarray*}
is positive definite on $\RR^d\times \RR^d$.
\end{lemma}

\begin{proof}For any $c_1,\ldots,c_n\in \RR$ and $u_1,\ldots,u_n\in\RR^d$, let $x_i=g(u_i)$. 
Then we have
\begin{eqnarray*}
\sum_{i,j=1}^{n}c_ic_j\tilde{\phi}(u_i,u_j)=\sum_{i,j=1}^{n}c_ic_j\phi(g(u_i),g(u_j))=\sum_{i,j=1}^{n}c_ic_j\phi(x_i,x_j)\geq 0,
\end{eqnarray*}
by the positive definiteness of $\phi$ in Lemma \ref{l22}. Then the lemma follows.
\end{proof}

\noindent \textbf{Acknowledgements.} {ZL is grateful for supports from NSF-1608896 and Simons-638143; FL is grateful for supports from NSF-1913243 and DE-SC0021361. FL would like to thank Yaozhong Hu, Chuangye Liu, Yulong Lu, Mauro Maggioni, Sui Tang and Cheng Zhang for helpful discussions.}

\bibliography{LearningTheory_abbr,ref_FeiLU,learning_dynamics_abbr,ref_stocParticleSys_abbr,PDNr}
\bibliographystyle{plain}

\end{document}